\documentclass[letterpaper, 10 pt, conference]{IEEEconf}
\IEEEoverridecommandlockouts 
\usepackage{geometry}
\usepackage{amsthm}      
\usepackage{amsmath,amsfonts}
\usepackage{mathptmx} 
\usepackage{amssymb}  
\usepackage{algorithmic}
\usepackage{algorithm}
\usepackage{array} 
\usepackage{textcomp}
\usepackage{stfloats}
\usepackage{url}
\usepackage{verbatim}
\usepackage{graphicx}

\usepackage{subfloat}

\usepackage{subcaption}
\usepackage[font=small,labelfont=bf]{caption}
\usepackage[utf8]{inputenc} 

\usepackage{tabularx}

\usepackage{siunitx}

\usepackage{bondgraphs}
\usepackage{tikz}
\usetikzlibrary{calc}
\usepackage{epstopdf}
\usetikzlibrary{shapes,fit}

\newtheorem*{note*}{Note}
\newtheorem*{theorem*}{Theorem}
\newtheorem*{objective*}{PBC objective}
\newtheorem*{tankobjective*}{Sampled PBC objective}

\usepackage[style=ieee,
			url=false, 
			isbn=false,
			doi=true,
			giveninits=true,
			uniquename=init,
			isbn=false,
			backend=bibtex,
			minbibnames=1,
			maxbibnames=6	]{biblatex}
\bibliography{main.bib}

\usepackage[colorlinks=true,
			urlcolor=black,
			citecolor=black,
			linkcolor=black	]{hyperref}

\begin{document}

\title{\LARGE \bf
Learning passive policies with virtual energy tanks in robotics  
}

\author{
Riccardo Zanella\textsuperscript{1}, 
Gianluca Palli\textsuperscript{1},
Stefano Stramigioli\textsuperscript{2},
Federico Califano\textsuperscript{2}
\thanks{\textsuperscript{1} Department of Electrical, Electronic and Information Engineering, University of Bologna, Viale Risorgimento 2, 40136 Bologna, Italy. Corresponding author email:
\textit{riccardo.zanella2@unibo.it} }
\thanks{\textsuperscript{2} Robotics \& Mechatronics (RaM) group, University of Twente, The Netherlands.}
\thanks{This work was supported by the European Commission’s Horizon Europe Programme with the project IntelliMan - AI-Powered Manipulation System for Advanced Robotic Service, Manufacturing and Prosthetics - under grant agreement No 101070136.}
}%
%
%

\markboth{IEEE Robotics and Automation Letters.}%
{Zanella \MakeLowercase{\textit{et al.}}: Learning passive control policies}
 
\maketitle

\begin{abstract}
Within a robotic context, we merge the techniques of \textit{passivity-based control} (PBC) and \textit{reinforcement learning} (RL) with the goal of eliminating some of their reciprocal weaknesses, as well as inducing novel promising features in the resulting framework.
We frame our contribution in a scenario where PBC is implemented by means of \textit{virtual energy tanks}, a control technique developed to achieve closed-loop passivity for any arbitrary control input. Albeit the latter result is heavily used, we discuss why its practical application at its current stage remains rather limited, which makes contact with the highly debated claim that passivity-based techniques are associated with a loss of performance. The use of RL allows us to learn a control policy that can be passivized using the energy tank architecture, combining the versatility of learning approaches and the system theoretic properties which can be inferred due to the energy tanks. Simulations show the validity of the approach, as well as novel interesting research directions in energy-aware robotics.
\end{abstract}


\section{Introduction}

Robotics is increasingly focusing on the development of control frameworks allowing the transition from industrial cages to unstructured environments. This transition carries the ambitious objective of stable and safe execution of complex tasks where robots coexist with other robots or possibly humans. The practical impossibility of careful dynamic modeling of the environment along those tasks makes the stability objective even more challenging, and model-based approaches poorly suitable.

Passive controllers have been presented as a feasible solution to tackle this problem as the stability of the closed-loop system is in principle independent on the external environmental interaction \cite{Stramigioli2015Energy-AwareRobotics,vanderSchaftL2}. Passive systems route physical energy rather than producing it, which is the ultimate reason why stability proofs (using energy functions as Lyapunov candidates) are conveniently assessed in this framework \cite{Duindam2009ModelingSystems,Ortega2004InterconnectionSurvey}. A powerful technique allowing the passivization of any control action is represented by \textit{virtual energy tanks} \cite{Califano2022OnSystems,Secchi2006PositionTelemanipulation,Duindam2004Port-basedSystems}. These are able to store information about physical energy flows undergoing the system, and introduce at a control level the scalar information representing the energy budget for the controlled robot. A conditional check on this budget with proper passive reaction strategies in case of depletion of the latter is what guarantees closed-loop passivity. The main drawback of passivity-based control methods (comprising those involving energy tanks) is the lack of optimization over a task performance metric along the design of the controller \cite{7384497}. This fact led to the claim that passivity-based control methods are associated with a loss of performance, which becomes more severe as the complexity of tasks increases.

On the other side of the control theoretic spectrum, recent advancements in the machine learning community are leading to robots with an outstanding awareness of complex environments and tasks. The intelligence encoded in the control policies, normally optimized by means of performance-based metrics, reflects complex high-level decision-making strategies which are learned thanks to the availability of huge datasets \cite{Zeng2018, Singh2019, Zhu2020}. The drawback of these families of approaches is the difficulty in guaranteeing system theoretic properties such as passivity and stability for the controlled system \cite{Busoniu2018ReinforcementApproximators}. 

In this work we merge the passivity-based control (PBC) design involving energy tanks and reinforcement learning framework (RL), combining system theoretic properties of the closed-loop system induced by the ultimately passive design, and high-performance achievement peculiar of the RL framework.
This procedure presented both in inference and in training, allows meaningful scale PBC to tasks requiring the learning of complex control policies.

\subsubsection*{Related work}
We recognize related work casting the energy tank algorithm (seen as a task-agnostic passivization tool) into a framework using task-based information for performance augmentation of the underlying passive system. In \cite{Shahriari2018Valve-basedObjectives} the authors introduce the so-called \textit{valve-based energy tanks} (also used in e.g., \cite{Benzi2022AdaptiveEstimation}), in which extra parameters are introduced on the energy tanks in order to embed control objectives in the design, beyond achieving passivity. Task specifications are translated into tank design rules by controlling the power flows undergoing the system. The idea of using a reference power trajectory in combination with energy tanks is present also in \cite{Shahriari2020PowerTanks,califano22}, where task-based specifications are used to tune the tank parameters. These works present very stiff task specifications and badly scale to complex tasks in which high-level policies need to be learned. 
A family of methods using energy tanks and sharing a similar motivation of this work is \cite{benzi22,Capelli2022PassivityEnergy}, where an explicit optimization problem is introduced to find the closest passive approximation of a given control action. The idea is to exploit the versatility of energy tank architecture to perform an optimization that can be generalized and scaled for different tasks and whose outcome is a passively controlled system. An important difference between this work and the cited ones (beyond the specific technique to solve the optimization) is that in \cite{benzi22,Capelli2022PassivityEnergy} the desired control action is already given, and the optimization aims at finding the closest passive approximation to it, imposed by a non-depletion constraint of the tank. We will comment on the difference between the approaches, which involves the degree of freedom represented by the initial state in the tank.

The paper is organized as follows. In Sec. \ref{sec:background} a throughout explanation of the PBC problem in robotics and its achievement using energy tanks is reviewed, followed by a discussion regarding the limitations of the methodology. In Sec. \ref{sec:framework} the proposed scheme is presented, and validated in Sec. \ref{sec:simulations} by means of different simulations. Conclusions and future works are sketched in Sec. \ref{sec:Conclusion}.

\section{Background and Motivation}
\label{sec:background}
We review, restricted to robotics, the motivation underlying PBC, and how energy tanks formally solve the problem of passivizing an arbitrary control input. We then discuss the limits of the approach, which are tackled with the subsequent RL integration.

\subsection{Passivity-based control and energy tanks}
Consider the dynamic  model of an $n$-DoF robot in Lagrangian joint-space coordinates:
\begin{equation}
\label{eq:robot}
    M(q)\ddot{q}+C(q,\dot{q})\dot{q}+D(q)\dot{q}+\frac{\partial V(q)}{\partial q}=\tau
\end{equation}
where $q\in \mathbb{R}^n$ are the joint coordinates, $M(q)=M^\top(q)>0$ is the inertia tensor, $C(q,\dot{q})$ is the matrix collecting centrifugal and Coriolis terms, $D(q)=D^\top(q)\geq0$ is the matrix collecting friction coefficients, $V: \mathbb{R}^n \to \mathbb{R}$ maps the joint coordinates into the total conservative potentials (e.g., gravity and elastic springs), and $\tau \in \mathbb{R}^n$ are the generalized forces at the joints, collocated with the joint coordinates $q$.
The time derivative of total mechanical energy $E(q,\dot{q})=\frac{1}{2}\dot{q}^\top M(q)\dot{q}+V(q)$ verifies the inequality
\begin{equation}
\label{eq:passivity}
    \dot{E}= \dot{q}^\top\tau-\dot{q}^\top D(q) \dot{q}\leq \dot{q}^\top\tau
\end{equation}
which is a statement of passivity for system (\ref{eq:robot}) with the energy function $E(q,\dot{q})$ as \textit{storage function}. Inequality (\ref{eq:passivity}) states that system (\ref{eq:robot}) is passive with respect to its energy function $E(q,\dot{q})$, and its so-called \textit{power port} $(\tau,\dot{q})$, an input-output pair whose pairing $\dot{q}^\top \tau$ produces the scalar power flow undergoing the system. 
Once (\ref{eq:passivity}) is integrated in time, the left-hand side of the equation $E(t)-E(0)$ represents the stored energy in the system, which is always less or equal than the supplied energy through the power port, represented by the right term $\int_0^t \dot{q}(s)^\top\tau(s) ds$. The positive dissipated energy $\int_0^t \dot{q}(s)^\top D(q(s)) \dot{q}(s) ds$ determines the convergence rate to the equilibrium, and acts as a passivity margin, i.e., the higher it is, the bigger is the margin for which the passivity inequality (\ref{eq:passivity}) results satisfied\footnote{If $D=0$, (\ref{eq:passivity}) holds with equal sign and the system is denoted as \textit{lossless}.}. If applied to an (open-loop) physical system like (\ref{eq:robot}), passivity just represents the thermodynamic statement of energy conservation, while if applied to a controlled system, passivity implies stability for the minimum of the storage function under weak conditions qualifying the latter as a Lyapunov function for the equilibrium \cite{vanderSchaftL2,Ortega2004InterconnectionSurvey}.
Next, we illustrate the full motivation behind the necessity of achieving a passive closed-loop system.

\subsubsection*{Passivity as must}
To understand the motivation behind a passive design consider system (\ref{eq:robot}) as an \textit{open} one, i.e., a system that, independently on the way it is controlled, can interact with a new, dynamically unknown system, that we call the \textit{s-environment}\footnote{We avoid calling this system simply "environment", since the latter will refer to the concept of environment in the RL framework.}. This is shown by considering system (\ref{eq:robot}) in which we distinguish the torque contributions in two terms, i.e.,
\begin{equation}
\label{eq:torqueDist}
\tau=\tau_c+\tau_e,
\end{equation}
where $\tau_c$ are the control torques that can be directly applied at the joints by means of the actuators, while $\tau_e=J^\top(q)F$ corresponds to the torques at the joints produced by an external interaction at the end-effector of the robot. Here $F\in \mathbb{R}^6$ is a vector collecting forces and torques at the end-effector in the workspace and $J(q)$ is the Jacobian matrix of the robot. The latter relates joint-space and work-space velocities and forces as the dual relations $\dot{x}=J(q)\dot{q}$ and $\tau_e=J^\top(q)F$, where $\dot{x}\in \mathbb{R}^6$ is the rate of work-space coordinates $x$, collocated with the forces $F$, resulting in $\tau_e^\top \dot{q}=F^\top \dot{x}$.
Specializing the power balance (\ref{eq:passivity}) to the uncontrolled system (\ref{eq:robot}) with the split (\ref{eq:torqueDist}) one obtains
\begin{equation}
\label{eq:dissipationsplit}
    \dot{E}=\dot{q}^\top\tau_c+\dot{x}^\top F -\dot{q}^\top D(q) \dot{q}\leq \dot{q}^\top\tau_c +\dot{x}^\top F,
\end{equation}
i.e., the robotic system (\ref{eq:robot}) is passive with respect to two different power ports: $(\tau_c, \dot{q})$, whose pairing is the power flow undergoing the actuators, and $(F,\dot{x})$, whose pairing is the power flow undergoing the s-environment, an external dynamical system with its own (unknown) dynamics. 
In this context, the PBC objective can be described as follows

\begin{objective*}
Given the system (\ref{eq:robot}) with the split (\ref{eq:torqueDist}) inducing power balance (\ref{eq:dissipationsplit}), design a control law on the port $(\tau_c,\dot{q})$, possibly by designing a controller as a dynamical system in its own right, with state $x_c$ and energy $V_c(x_c)$, such that for the closed-loop storage function $\mathcal{E}(q,\dot{q},x_c)=E(q,\dot{q})+V_c(x_c)$, the passivity condition
\begin{equation}
\label{eq:passivityEnv}
    \dot{\mathcal{E}}\leq\dot{x}^\top F
\end{equation}
is verified, i.e., that the controlled system is passive with respect to the s-environment port $(F,\dot{x})$.
\end{objective*}

This objective is considered of utmost importance (from which the \textit{passivity as must} claim \cite{Stramigioli2015Energy-AwareRobotics}) since, when an interaction takes place ($F\neq 0$) the dynamics of the closed-loop system deeply depends on the dynamics of the s-environment, which cannot be modeled in an oversimplified way, e.g., assuming linearity. This claim is particularly valid if the task to be executed takes place in an unstructured, hazardous scenario where sources of disturbances like unforeseen interactions with humans are possible. 

The passive design (\ref{eq:passivityEnv}) aims at \textit{excluding the possibility that in the moment the robot interacts with a passive s-environment, an unbounded amount of energy can be produced during the interaction}.\footnote{We refer to \cite{Stramigioli2015Energy-AwareRobotics} for a converse proof: if PBC is not achieved, there exists a passive s-environment destabilizing the controlled system.} This property is referred to as \textit{contact stability}. For completeness, we report that passivity induces also a stronger property than contact stability, which is sometimes referred to as \textit{robust stability} \cite{benzi22}: even if the s-environment is not a passive system, contact stability can be assured if the passivity margin of the controlled system covers the energy generated by the active s-environment. All these properties follow from the fact that passive systems are closed under power-preserving interconnection \cite{vanderSchaftL2,Duindam2009ModelingSystems}, and we refer to \cite{vanderSchaftL2} for further relations between passivity and different system theoretic stability properties. To conclude, in this work we are aligned with the necessity of achieving the PBC objective when a robotic task needs to be performed.
\subsubsection*{Energy Tanks}
An extremely effective method to achieve the PBC objective relies on the energy tank control algorithm, which we describe next and whose basic schematics is represented in Fig. \ref{fig:energy_tank}.
\begin{figure}
    \centering
     \includegraphics[width=0.65\columnwidth]{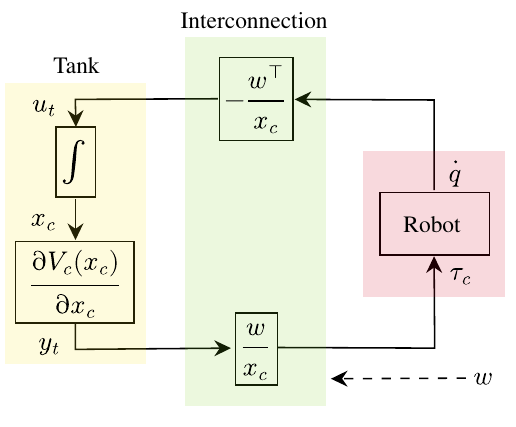}
    \caption{The essential of the energy tank algorithm.}
    \label{fig:energy_tank}
     \vspace{-3mm}
\end{figure}
Mathematically an energy tank is a dynamical system that constitutes an atomic energy-storing element. It can be represented as the (lossless) $1$D system:
\begin{align}
\label{eq:tank1}
    \dot{x}_c&=u_{\textrm{t}}  \\
    \label{eq:tank2}
        y_{\textrm{t}}&= \frac{\partial V_c(x_c)}{\partial x_c},
\end{align}
where the energy function is simply $V_c(x_c)=\frac{1}{2}x_c^2$, and as a consequence it presents a power port $(u_{\textrm{t}},y_{\textrm{t}})$ since $\dot{V}_c=y_{\textrm{t}} u_{\textrm{t}}$. The tank (\ref{eq:tank1}-\ref{eq:tank2}) is interconnected to  (\ref{eq:robot}) in a way that allows the implementation of some control action that fulfills the execution of some task, and at the same time meeting the desired passivity constraint. This is possible thanks to a suitable power-preserving interconnection between the two systems, defined as follows:

\begin{equation}
\label{eq:interconnection}
    \begin{pmatrix} \tau_c \\ u_{\textrm{t}} \end{pmatrix} = \begin{bmatrix} 0 & \frac{w}{x_c} \\ -\frac{w^\top}{x_c} & 0\end{bmatrix} \begin{pmatrix} \dot{q} \\ y_{\textrm{t}} \end{pmatrix},
\end{equation}
where $w\in \mathbb{R}^n$ is the desired task-dependent control action to be passively implemented.
This interconnection, which is the core of the energy tank algorithm, produces two effects: i) the desired action $w$ is correctly implemented, i.e., from the side of the robot (\ref{eq:robot}), one obtains
\begin{equation}
    M(q)\ddot{q}+C(q,\dot{q})\dot{q}+D(q)\dot{q}+\frac{\partial V(q)}{\partial q}=w+J^\top(q)F;
\end{equation}
and ii) at every time the mechanical power exiting the robot produced by the actuators $\dot{q}^{\top} \tau_c$ leaves the energy tank since
\begin{equation}
\label{eq:tankBalance}
    \dot{V}_c=y_{\textrm{t}} u_{\textrm{t}}=-w^\top \dot{q}=-\dot{q}^\top w=-\dot{q}^\top \tau_c.
\end{equation}
Evaluating the variation of the closed-loop storage function $\mathcal{E}=E+V_c$ one achieves the PBC objective (\ref{eq:passivityEnv}), which shows that the interconnection achieves indeed a passive closed-loop system with respect to the s-environmental port.
The scalar $V_c(x_c)$ indicates the amount of energy that is still at disposal of the control mechanism implementing the action $w$ before losing the described formal passivity. In fact, the interconnection (\ref{eq:interconnection}), which is modulated by the tank state, becomes singular when $x_c$, and thus $V_c(x_c)$ is zero, representing the moment in which it becomes impossible to passively perform the desired action $w$. To meaningfully implement the energy tank algorithm, it is thus necessary to constantly observe the energy in the tank in order to implement a control action $\Bar{w}$, rather than $w$, where

\begin{equation}
     \Bar{w}= 
\begin{cases}
    w,& \text{if } V_c(x_c)\geq \epsilon \geq 0\\
    0,& \text{if } V_c(x_c)\leq \epsilon ,
\end{cases}
\label{eq:tanksCheck}
\end{equation}
for some small positive energy level $\epsilon$. In this way, formal passivity of the closed-loop system is recovered completely since the system is just detached from the controller at the moment in which no energy budget is left.

Before discussing the peculiarities and limitations of the described algorithm, we introduce the concept of \textit{task energy}, that we will denote by $e^*$, firstly defined in  \cite{Schindlbeck2015UnifiedTanks}, which is the minimum energy in the tank necessary to fulfill passively some task. In other words, if $V_c(x_c(t))|_{t=0}>e^*$, the tank never depletes along the whole task horizon. It is important to notice that $e^*$ does not depend on the task only, but also on the specific tank dynamics which is chosen: the choice $u_\textrm{t}$ in (\ref{eq:interconnection}) is not unique, and many variations have been proposed \cite{Califano2022OnSystems} (e.g., recirculation in the tank of dissipation) which would change the value of task energy for the same task. 
\subsubsection*{Comments and limitations} 
i) The energy tank algorithm constitutes a clever way to disembody from any physical dynamics the implementation of a passive control action. In fact, there is no need to design the controller in a way that it mimics a physical system (e.g., in the PD case the controller can be seen as a control spring and damper): as long as a task is achieved through a control input $w$ and $e^*$ is finite, a passive implementation is possible by means of the energy tank algorithm, just setting $V_c(x_c(t))|_{t=0}>e^*$; ii) The aforementioned initial energy assignment is a degree of freedom in the algorithm, whose implications are often naively addressed. In fact, a very high (yet bounded) energy initialization in the tank would technically still fulfill the PBC objective (\ref{eq:passivityEnv}), yet creating so-called "practically unstable behaviors" \cite{benzi22,Califano2022OnSystems}. As a consequence, a naive tank design \textit{de facto} makes robust stability a property that is not connected to any safety guarantee. In fact, notice that till the energy in the tank is not depleted, the control input $w$ is completely transparent to the tank algorithm, which reduces to a trick to formally prove passivity, with limited significance in the context in which tasks need to be performed in unstructured scenarios. The fact that the mechanical power flow is undefined in sign (i.e., when $\dot{q}^\top \tau$ the tank gets refilled) worsens this criticality, which is sometimes addressed with empirical tank saturation arguments.

To conclude the discussion, both the task energy $e^*$ and the control action $w$ are often difficult to be determined a priori, and these are not independent variables. For complex task executions, it is reasonable to take advantage of simulations to optimize for both variables \textit{in a combined way}. If the PBC objective can be naively achieved just by initializing the energy in the tank to a sufficiently high value, what we claim to be a useful system theoretic property in the energy tank context is the achievement of the PBC objective \textit{combined} with an estimation of the task energy. In fact, the knowledge of the latter would lead to a meaningful energy tank initialization, so that a depletion of the tank can be used as a diagnostic tool to detect an important divergence from nominal task execution, beyond formally achieving a passive closed-loop system.

\subsection{Reinforcement learning} 
Reinforcement learning (RL) \cite{sutton2018reinforcement} is a model-free framework, consisting of an agent interacting with an environment, to solve optimal control problems stated as Markov decision processes (MDPs). MDPs are a mathematical formulation of decision-making in situations where outcomes are partly stochastic and partly under the control of a decision-maker.
 A MDP is defined by the tuple $(\mathcal{S},\mathcal{A}, p, r)$:  at any given state $s_k\in\mathcal{S}$ at time step $k$, the agent chooses and executes an action $a_k\in \mathcal{A}$ according to a learnt policy $\pi(a_k|s_k)$, then the environment transitions to a new state $s_{k+1}\in \mathcal{S}$ with the unknown state transition probability $p(s_{k+1}|s_k,a_k)$ and produces a reward $r_k = r(s_k, a_k)$. An important condition that characterizes the MDPs is that the state transition probability satisfies the Markov property $p(s_{k+1}|s_k,a_k)=p(s_{k+1}|s_k,a_k,\dots, s_1,a_1)$ for any trajectory $s_1,a_1,\dots,s_k,a_k$, which means that the environment is memoryless because the transition to the next state depends only on the current state and action. 
The Markov property is usually fundamental to guarantee the convergence to the optimal policy in RL algorithms   \cite{sutton2018reinforcement}.
The goal of an RL algorithm is to find an optimal policy $\pi^*$ that maximizes the expected $\gamma$-discounted sum of future rewards (that we define as return). 

While the proposed framework could be combined with any RL algorithm, benchmarking is out of the scope of this work, for this reason in the rest of the paper we only consider the off-policy soft actor-critic (SAC) algorithm \cite{Haarnoja2018} with continuous action spaces.  

\section{Energy tanks meet RL}
\label{sec:framework}

The arbitrariness of the control input $w$ and its disembodiment from any physical dynamics are tempting motivations to choose it as a decision variable in an optimization framework. Indeed, the output of an RL control policy can be directly mapped to the torques $w$ commanded at an actuator level, as shown in Fig. \ref{fig:mainscheme}. This map can be trivially the identity (in case the generalized forces are directly learned in the RL scheme) or be represented by a transducer (e.g., an internal PID controller that maps a position or velocity reference to the commanded torques $w$), which is often shown to critically improve sample efficiency in the RL framework \cite{Smith2022}. No specific transducer between the action space and $w$ is required to implement the proposed scheme, as long as the commanded torques $w$ are accessible. 
Since MDPs are discrete-time processes, we start by exploiting a simple but powerful result that allows to reproduce the energy tank algorithm at the discrete-time level. 
\subsubsection*{The energy sampling approach}
Let us denote $e_{k}$ the level of energy in the tank at the discrete time step $k$, i.e., $e_k=V_c(kT)$.
The energy in the tank at the next time step, initialized by $e_0$, is updated according to the rule
\begin{equation}
    \label{eq:etank_update}
     e_{k+1} = e_k - \Delta e_{k+1} 
\end{equation}
where $ \Delta e_{k+1}$ is the amount of energy exiting the tank at time $k+1$, which approximates, at a discrete-time level, the energy leaving the tank in a sampling interval with duration $T$, i.e., $\int_{Tk}^{T(k+1)} -\dot{V}_c (s)\,\, ds$. With a look at the continuous energy balance (\ref{eq:tankBalance}), this energy equals the sum of energies used by each actuator of the robot in the time interval $[Tk,T(k+1))$:
\begin{equation*}
    \int_{Tk}^{T(k+1)} -\dot{V}_c (s)  ds= \int_{Tk}^{T(k+1)} w(s)^\top \dot{q} (s)  ds.
\end{equation*}
Now, assuming the torque signal is constant with value $w_k$ along the $k$-th sampling interval, and indicating with $q_k:=q(Tk)$, we can further massage the previous expression to define $\Delta e_{k+1}$ as:
\begin{equation}
\label{eq:defek}
    w_k^\top \int_{Tk}^{T(k+1)}\dot{q}(s)  ds= w_k^\top (q_{k+1}-q_{k})=:\Delta e_{k+1}.
\end{equation}

It is worth remarking that, in the common conditions of constant torque along the sampling time interval and position sensor collocated to the motor, such defined quantity produces the \textit{exact} amount of injected energy by the actuators independently of the value of $T$, and not an estimate. This property makes the connection between the discrete-time RL framework and the energy tank framework particularly appealing since the tank algorithm can exactly be reproduced at a discrete time level without the need of integrating tank dynamics (\ref{eq:tank1}) with the input defined in (\ref{eq:interconnection}).

The definition of $\Delta e_{k+1}$ in (\ref{eq:defek}) exactly reproduces at a discrete time level the tank algorithm as reviewed in Sec. \ref{sec:background}, but this choice is not unique, and variations are possible on the basis of the passivity margin that one wants to achieve for the closed-loop system.
In the sequel of this work, we want to prevent refilling of the tank, and thus use the following update rule
\begin{equation}
\label{eq:deltaenergy}
     \Delta e_{k+1} :=   \begin{cases}
                   w_k^\top (q_{k+1}-q_{k}) & \text{ if }  w_k^\top (q_{k+1}-q_{k}) \geq 0 \\
                   0 & \text{ otherwise }
                \end{cases}
\end{equation}
Such an update rule physically represents the impossibility of recovering negative mechanical energy flowing into the system, which is a reasonable assumption for many mechanical systems. Furthermore, this choice serves best our purpose of interpreting the task energy as a metabolic metric to assess an initial energy budget: the energy in the tank can only decrease along the task execution. 

In this discrete-time implementation of energy tanks, we aim to achieve the following finite difference version of the PBC objective (\ref{eq:passivityEnv}). We indicate $E_k=E(kt)$ the sampled mechanical energy of the robot and $\mathcal{E}_k=E_k+e_k$ the sampled closed-loop storage function.

\begin{tankobjective*}
Given the system (\ref{eq:robot}) with the split (\ref{eq:torqueDist}) inducing power balance (\ref{eq:dissipationsplit}), design a control policy inducing desired control action $\tau_c=w_k$
such that the condition
\begin{equation}
\label{eq:passivityEnvDisc}
    \mathcal{E}_{k+1}-\mathcal{E}_{k}\leq\int_{Tk}^{T(k+1)}\dot{x}(s)^\top F(s) ds
\end{equation}
is verified $\forall k$, i.e., that the sampled difference of closed-loop storage function is bounded by the energy injected in a sampling interval by the s-environment through the port $(F,\dot{x})$.    
\end{tankobjective*}
\begin{theorem*}
    Assuming constant generalized forces $w_k$ in the sampling interval $[Tk,T(k+1)]$, the PBC objective (\ref{eq:passivityEnvDisc}) is achieved using the tank update rule (\ref{eq:deltaenergy}).
\end{theorem*} 
\begin{proof}
    Define $R_{k+1}:=\int_{Tk}^{T(k+1)}\dot{q}(s)^\top D(q(s)) \dot{q}(s) ds$ the dissipated power by friction in a sampling interval. Using (\ref{eq:etank_update}) and integrating (\ref{eq:dissipationsplit}) in a sampling interval, we compute $\mathcal{E}_{k+1}-\mathcal{E}_{k}=E_{k+1}-E_{k}+e_{k+1}-e_{k}=-R_{k+1}+w_k^\top (q_{k+1}-q_{k})+\int_{Tk}^{T(k+1)}\dot{x}(s)^\top F(s) ds-\Delta e_{k+1}$, where we used the fact that $\tau_c=w_k$ is constant in a sampling interval.
    The final claim (\ref{eq:passivityEnvDisc}) follows from (\ref{eq:deltaenergy}) and $R_{k+1}\geq 0$.
\end{proof}

Note that the update rule (\ref{eq:deltaenergy}) induces a stronger passivity margin with respect to (\ref{eq:defek}) since when $\Delta e_{k+1}=0$ the term $w_k^\top (q_{k+1}-q_{k})$ is negative, and as such acts as dissipation.

We define the \textit{energy spent at step $k$} as the sum of the energies exiting the tank up to that step:
\begin{equation}
\hat{e}_k := \sum_{l=0}^{k-1}\Delta e_{l+1}.
\end{equation}
Since with the update rule (\ref{eq:deltaenergy}) the sequence $\hat{e}_k$ is monotonically non decreasing, the task energy $e^*$ over a task with $N$ time steps can be simply calculated as  
\begin{equation}
\label{eq:taskenergy}
    e^* = \hat{e}_N.
\end{equation}

We describe now two procedures combining the tank and the RL algorithm. The first one aims at passivizing a control policy that was previously learned agnostically to any passive design. The second exploits the tank architecture also in the training phase, and produces by construction a passive learned policy. 

\begin{figure}[t]
    \centering
    \includegraphics[width=0.8\columnwidth]{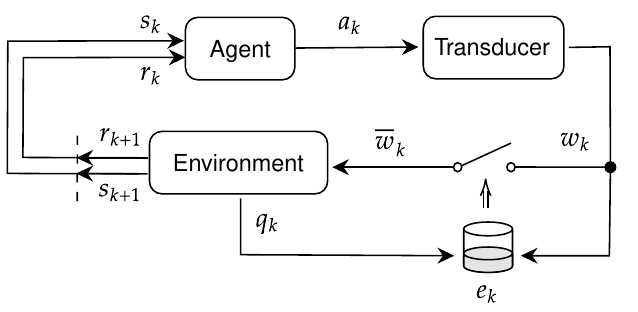}
    \vspace{-1mm}
    \caption{RL with energy tank passivization}
    \label{fig:mainscheme}
    \vspace{-5mm}
\end{figure}

\subsection{ Passivization in Inference (Passivizing learned policies)} \label{sec:passiveinference}

We indicate with the term inference the phase where the training is finished and the agent implements the learnt policy, without exploring anymore. 
The combination of the energy tank architecture and the reinforcement learning technique in inference is quite straightforward, as an RL agent represents in this context an arbitrary controller that the tank algorithm is able to passivize. In fact, we can use any previously trained agent in a generic environment and passivize the controlled system in inference by wrapping the control action with (\ref{eq:tanksCheck}), where the energy in the tank is updated according to (\ref{eq:etank_update}) and (\ref{eq:deltaenergy}). Furthermore, an agent able to fulfill a given task in a generic environment with a specific task energy $e^*$, is also able to passively fulfill that task in an environment where the control action is wrapped with (\ref{eq:tanksCheck}) and the energy tank is initialized with a value greater than $e^*$. This represents indeed the only design requirement to preserve the agent's performance in executing the task with the learnt policy \textit{while} achieving formal passivity (\ref{eq:passivityEnvDisc}).
In order to equip the controlled system with a \textit{meaningful} passivity property, it is important to initialize the tank with an amount of energy that slightly exceeds $e^*$, which can be estimated by running the agent without (\ref{eq:tanksCheck}) for a number of episodes and measuring the maximum consumption of energy along all the episodes. The proposed procedure provides a constructive way to achieve the PBC objective and brings the additional benefit of task energy estimation, which produces a meaningful closed-loop passivization. This step provides a diagnostic tool indicating a non-regular task execution in case of tank depletion, which is a piece of extra information that a naive tank design lacks.

\begin{note*}
It is meaningful to compare the described procedure with the works  \cite{benzi22,Capelli2022PassivityEnergy}, which aim at passivizing an already given control input, and as such can also be seen as a passivization in inference algorithm. In \cite{benzi22,Capelli2022PassivityEnergy} the task energy is not explicitly estimated and as a consequence the initialization of the tank (whose non-depletion constraint in an optimization problem is what produces the passive approximation) is arbitrary. By solving the optimization with different tank initialization, in \cite{benzi22} different passive approximations are achieved. Here instead, motivated by the fact that passivity can be achieved anyhow for any finite task energy and that the learnt policy represents the way the task is optimally executed, we first estimate the task energy to constraint the degree of freedom of tank initialization. \end{note*}  
 
\subsection{Passivization in Training (Learning passive policies)} \label{sec:passivetrain}

In many circumstances, it might be desirable to continuously achieve a passive closed-loop system also in the training phase, i.e., when the task-dependent control policy is being learned. This is desirable for instance when the training takes place in real life, and not in simulations, such that initial exploration phases are guaranteed not to undergo hazardous unbounded energy generation. Furthermore, the passivized training phase endows the agent with awareness of the metabolic spending encoded in the tank architecture, induced by the initial energy budget $e_0$. As a consequence, the learned policy will be influenced by the tank initialization $e_0$, and in particular, the task energy $e^*$ will be directly learned together with the control policy in a combined way, strengthening the significance of the resulting passivity property. This mechanism provides a clear biomimetic perspective to the scheme since the energy budget and the task-dependent policy are not independent variables.

The proposed procedure is achieved by implementing the learning in the extended RL framework depicted in Fig. \ref{fig:mainscheme}., i.e., the standard RL scenario with the tank algorithm embedded. We recognize the following technical issue which needs to be addressed before implementing the training.
\subsubsection*{Loss of Markovian property} 

Notice that the current value of the tank state $e_k$ depends on the entire state-action trajectory of the system from the beginning to step $k$. This is true due to the update rule (\ref{eq:etank_update}), where $\Delta e_{k+1}$ depends both on the joint position $q_k$ (which in most robotics applications is part of the RL state $s_k$) and the action $a_k$.

Furthermore the value $e_k$ is able, due to (\ref{eq:tanksCheck}), to influence the future dynamics of the environment in case of tank depletion.
For these reasons the new environment that the agent perceives by including also the energy tank architecture as in Fig. \ref{fig:mainscheme} does not preserve the memorylessness of the original environment, or in other words it loses the Markov property.

If the Markov property is not satisfied, most RL algorithms lose formal proof of convergence to the optimal policy, and long-term prediction performance can degrade when the one-step predictions defining the Markov property become inaccurate \cite{Whitehead1995}.
We identify two solutions to restore the Markov property.

\subsubsection*{Extended State}
A common practice when the environment is non-Markovian because a relevant time-dependent part of the information is missing from the agent's state but accessible is simply to include it in the state \cite{Whitehead1995}. However, under this formulation, the agent might require a dedicated tuning process to maintain similar performance.  

\subsubsection*{Extended Termination} As an alternative, the episode's termination condition can be extended to the situation in which the energy in the tank depletes.   
In case a robotic platform is involved, terminating the episode would mean breaking the joints, instead of giving zero torque as done in the original formulation. After the termination, the environment is reset and a new episode can start with a restored level of energy in the tank. In this formulation, the Markov property is restored without changing the state and the reward. In fact, the termination of the episode removes the switch behavior (\ref{eq:tanksCheck}), and the agent never meets a state influenced by the free dynamics. 

The reward function has to be non-negative, since the length of an episode is not fixed. 
In fact, using a negative reward, the agent might attempt to maximize the cumulative return by learning actions leading to tank depletion. 


\section{Simulations}
\label{sec:simulations}

\begin{figure}[t]
    \centering 
    \includegraphics[width=.99\columnwidth]{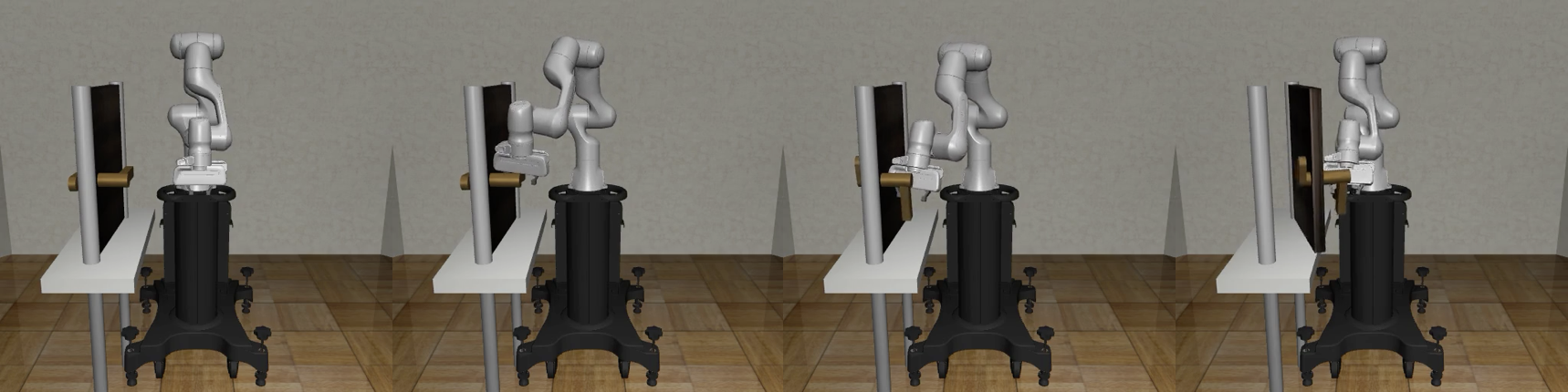}
    \caption{\texttt{DoorOpening} environment}
    \label{fig:envs/dooropening} 
    \vspace{-3mm}
\end{figure}

We consider two environments implemented in MuJoCo physics simulator \cite{Todorov2012MuJoCo:Control}: 
\subsubsection{\texttt{DoorOpening}} Here,  we adopted an  implementation from Robotsuite simulation framework \cite{Zhu2020} where a 7-DoF robotic arm must learn to open a door by turning the handle as shown in Fig. \ref{fig:envs/dooropening}. The door location is randomized at the beginning of each episode. The robot is a Panda Franka Emika which mounts a parallel-jaw gripper equipped with two small finger pads. 
The torques provided to the 7 joints of the robot are generated using a proportional control law $\tau=k_p(\dot{q}_d-\dot{q})$, with $k_p=0.03$, that follow a target joint velocity $\dot{q}_d$ provided by the agent at $50Hz$. 
\subsubsection{\texttt{Pendulum}} In this environment we created a pendulum composed of a rod suspended by one extremity from a pivot actuated by a motor controlled in torque with a control frequency of $50Hz$. The rod is subject to gravity while friction loss is present in the joint.  
Let us denote the joint angle as $\beta$. In the reference configuration $\beta  = 0$ the pendulum heads horizontally right. 
At the beginning of each episode, the angle is randomly initialized $\beta\sim\mathcal{N}(-\pi/2,0.05\pi)$.
The environment state is defined as the vector $s_k = \begin{bmatrix} \sin\beta_k & \cos \beta_k & \tanh \dot{\beta}_k \end{bmatrix}^\top $, while the agent action corresponds to the torque applied to the joint. 
Since the inverted pendulum problem consists of bringing the rod to the inverted position ($\sin\beta^*=1$) and holding it there as long as possible, we design a positive reward function $ r_k  = (1 + \left| \sin\beta_k -1 \right| + 0.1\left| \tanh \dot{\beta}_k \right| + 0.01\left| \tau_k \right|)^{-1}$ which is inversely proportional to a weighted sum of the absolute values of position error ($\sin\beta_k -1$), velocity error ($\tanh \dot{\beta}_k$) and torque ($\tau_k$). These last two components help to remove oscillatory behaviors in the transition phase and at the equilibrium point.

\begin{table}[]
    \caption{Hyperparameters}
    \label{tab:hyperparams}
    \centering
    \begin{tabular}{ccc}
        & {\scriptsize \texttt{DoorOpening}}  & {\scriptsize \texttt{Pendulum}} \\ \hline
        Actor Learning Rate     & 0.01             & 0.005            \\
        Critic Learning Rate    & 0.0005           & 0.005            \\
        Soft Update Coefficient & 0.005            & 0.003            \\
        Batch Size              & 128              & 256              \\
        Number of Epochs        & 2000             & 200              \\
        Steps per Epoch         & 2500             & 2500             \\
        Steps per Trajectory    & 500              & 500              \\
        Steps Before Training   & 3300             & 2500             \\
        Target Update Period    & 5                & 5                \\
        Gradient Steps          & 1000             & 500              \\
        Neurons per Layer       & [256,256]        & [256,256]   
    \end{tabular}
    \vspace{-3mm}
\end{table}

In all the simulations presented in this work the agents are trained using SAC algorithm implemented in PyTorch and trained with an NVIDIA GeForce GTX 1080 Ti on an Intel Core i7-7700K CPU clocked at 4.20GHz.  
In Table \ref{tab:hyperparams} the training configurations for both environments are reported. Additionally, generalized State-Dependent Exploration (gSDE) is employed, where a new noise matrix is sampled every episode. The entropy regularization coefficient is learned automatically as done in  \cite{Haarnoja2018SoftApplications}. 
 
In the rest of this section, we expose the results obtained from the simulations run on the two environments introduced above. In particular, the passivization of a policy learned on the \texttt{DoorOpening} and the learning of a passive policy on \texttt{Pendulum} are discussed respectively in  \ref{sec:simulations/passiveinference} and \ref{sec:simulations/passivetrain}. In Table \ref{tab:notation} the notation adopted for each agent in training and inference is briefly reported. 

\begin{table}[]
\caption{}
\label{tab:notation}
\begin{tabularx}{\columnwidth}{p{0.1\columnwidth}cX}   
\hline
$\phi_\infty$          & $\overset{\underset{\mathrm{\Delta}}{}}{=}$ & agent trained without passivization                                               \\
$\phi_{e_0}$           & $\overset{\underset{\mathrm{\Delta}}{}}{=}$ & agent trained with passivization and tank initialized at $e_0$                    \\
$\phi_{\infty|\infty}$ & $\overset{\underset{\mathrm{\Delta}}{}}{=}$ & agent $\phi_\infty$ in inference without  passivization                           \\
$\phi_{\infty|e^*}$    & $\overset{\underset{\mathrm{\Delta}}{}}{=}$ & agent $\phi_\infty$ in inference with passivization and tank initialized at $e^*$ \\
$\phi_{e_0|e^*}$       & $\overset{\underset{\mathrm{\Delta}}{}}{=}$ & agent $\phi_{e_0}$ in inference with passivization and tank initialized at $e^*$  \\
$\phi_{\infty|\infty,\delta}$ &
  $\overset{\underset{\mathrm{\Delta}}{}}{=}$ &
  agent $\phi_\infty$ in inference with an external force field and without passivization \\
$\phi_{\infty|e^*,\delta}$ &
  $\overset{\underset{\mathrm{\Delta}}{}}{=}$ &
  agent $\phi_\infty$ in inference with an external force field, passivization, and tank initialized at $e^*$ \\ \hline
\end{tabularx}
\end{table}

\subsection{Passivization in Inference}
\label{sec:simulations/passiveinference}

\begin{figure}[t]
    \centering 
    \includegraphics[width=.8\columnwidth]{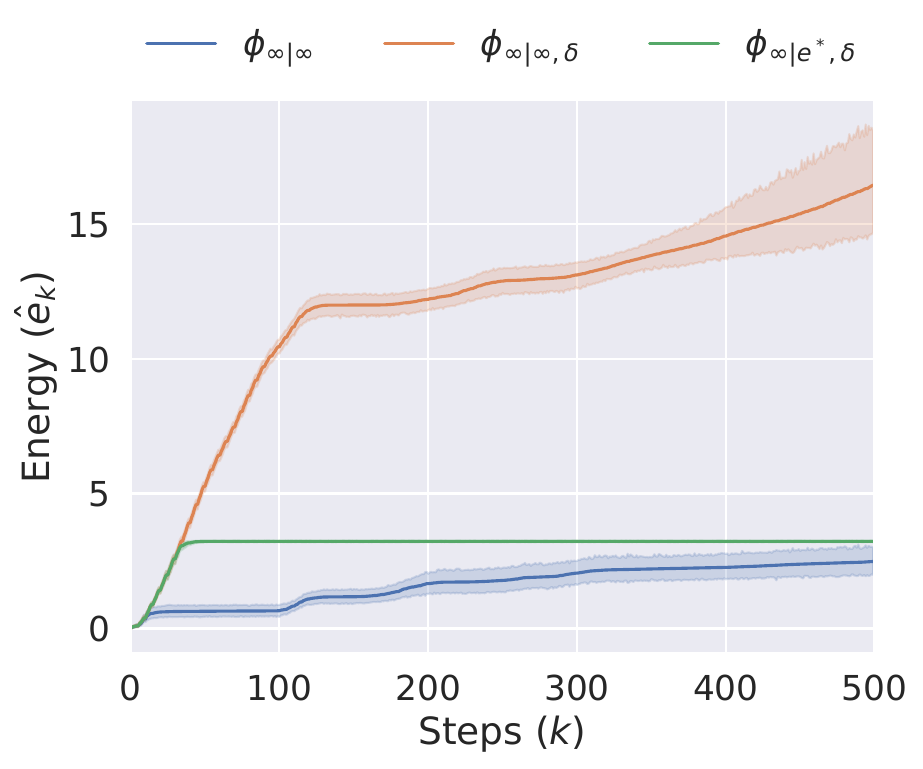}
    \caption{Average energy spent in \texttt{DoorOpening}  in inference. }
    \label{fig:inference_dooropening_wind} 
    \vspace{-3mm}
\end{figure}

To study the effects of passivization in inference, an external field of force is applied to the system as a way to perturb the nominal task. The magnitude of the force is chosen in order to let the tank deplete when it is initialized with an amount of energy corresponding to the task energy measured in the nominal case.  
For this experiment, we consider the \texttt{DoorOpening} environment and we employ a pre-trained model available online\footnote{https://github.com/ARISE-Initiative/robosuite-benchmark}. 
%
As visible in Fig. \ref{fig:inference_dooropening_wind}, under the effect of the external force the agent without the passivization in inference $\phi_{\infty|\infty,\delta}$ can spend an unbounded amount of energy, while the consumption is limited when the same agent is wrapped with passivization $\phi_{\infty|e^*,\delta}$. The energy tank is initialized with the task energy $e^*$ measured as the maximum $\hat{e}_k$ over the $100$ episodes in the agent without the external force and without passivization $\phi_{\infty|\infty}$.  
From this simulation, we can better comprehend and appreciate the versatility of the proposed framework. In fact, an agent trained by a third party to optimize the process in a non-passive way is readily passivized by wrapping the decision variable as detained in \ref{sec:passiveinference} and without the need to retrain/modify the agent.



\subsection{Passivization in Training}
\label{sec:simulations/passivetrain}

\begin{figure*}[t]
    \centering
    \subfloat[Return (training)]{
        \centering
        \includegraphics[width=.5\columnwidth]{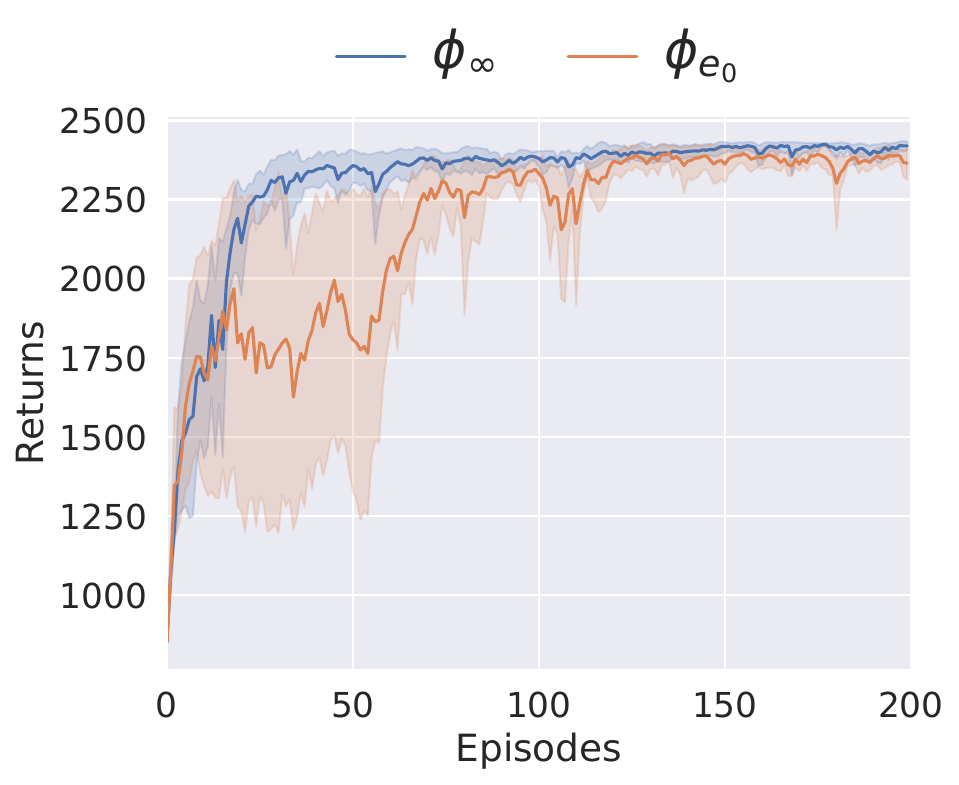}
        \label{fig:training_return_pendulum}
    } \hspace{-5mm}
    \subfloat[Energy (training)]{
        \centering
        \includegraphics[width=.5\columnwidth]{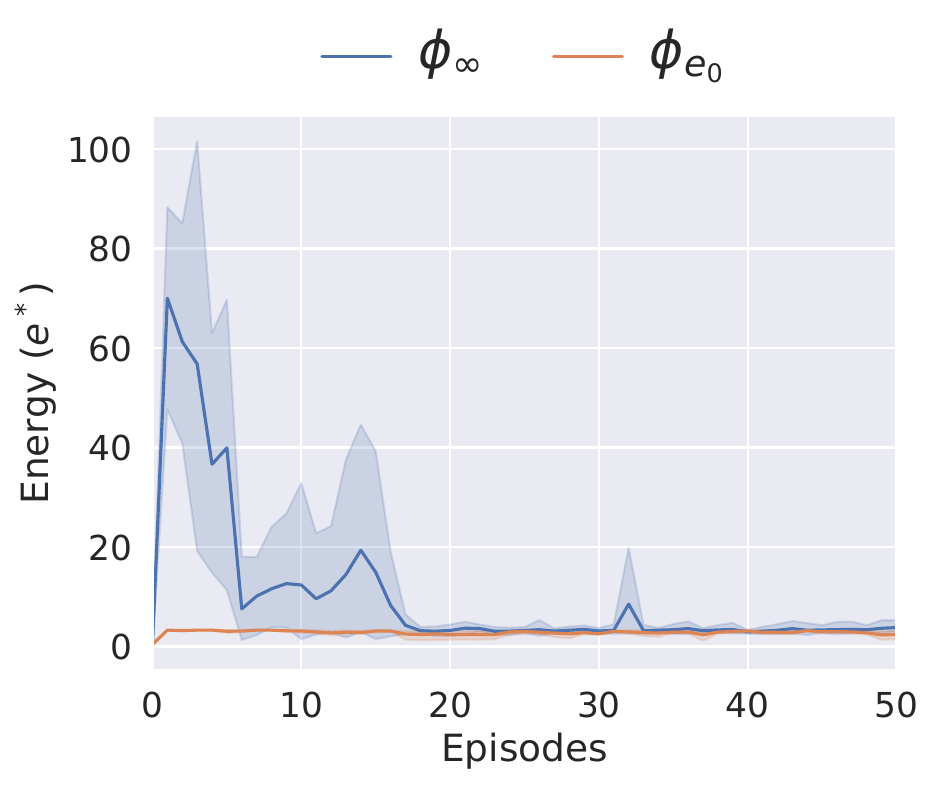}
        \label{fig:training_energy_pendulum}
    }  \hspace{-5mm}
    \subfloat[Error (inference)]{
        \centering
        \includegraphics[width=.5\columnwidth]{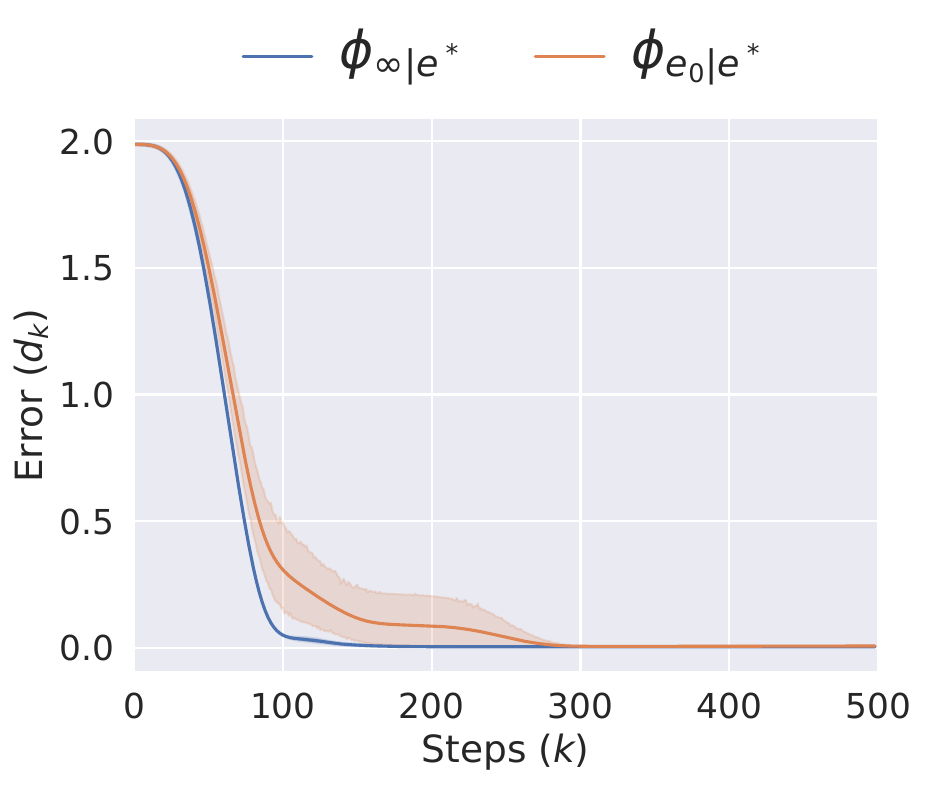}
        \label{fig:inference_error_pendulum}
    } \hspace{-5mm} 
    \subfloat[Energy (inference)]{
        \centering
        \includegraphics[width=.5\columnwidth]{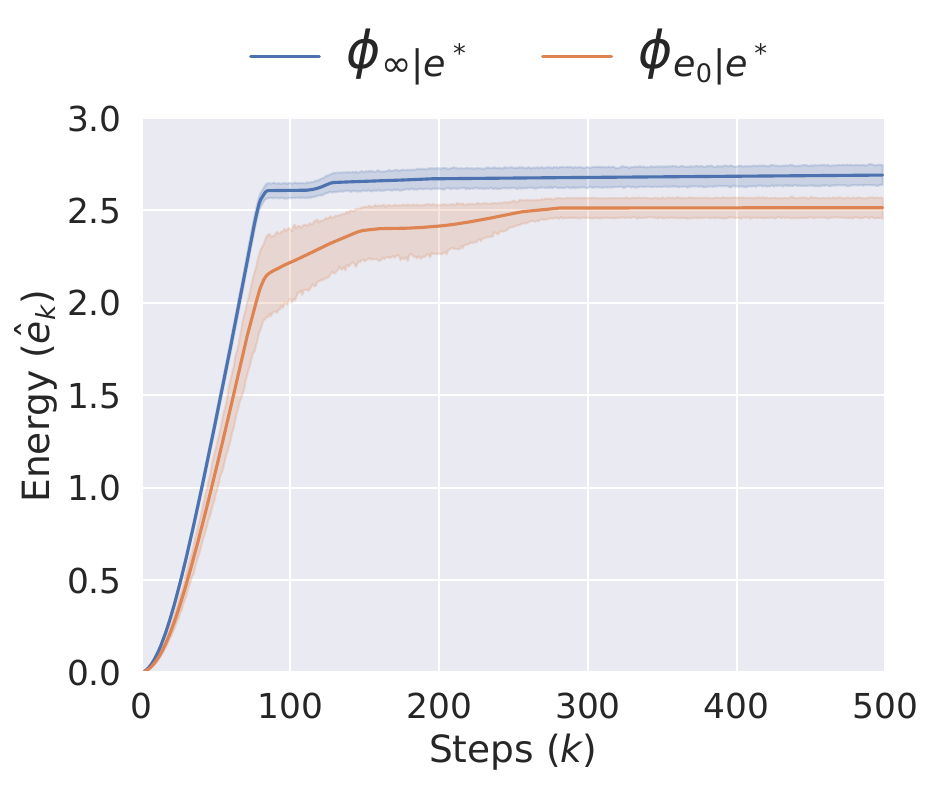}
        \label{fig:inference_energy_pendulum}
    } 
    \caption{In (a) and (b) respectively the average returns obtained and the energy spent in training. In (c) and (d) respectively the average position error  and energy spent in inference. 
    }
    \label{fig:passive_training} 
    \vspace{-3mm}
\end{figure*}

Let us consider two agents $\phi_{\infty}$  and $\phi_{e_0}$ trained in the \texttt{Pendulum} environment sharing the same hyperparameter configuration, but $\phi_{e_0}$ is passivized with the framework of Sec. \ref{sec:passivetrain} and the \textit{Extended Termination} method. 
For each simulation on the \texttt{Pendulum}, 5 instances with different random seeds for the pseudo-random generators are trained. 
Implementing the training with a sufficiently low $e_0$ provides some level of meaningful robust stability in the training phase. Furthermore, we aim to analyze how the energy budget $e_0$ imposed in training influences the learned policy and the task energy $e^*$, which are correlated variables.

We choose as $e_0$ the task energy estimated from $\phi_{\infty}$, measured in inference as the maximum $\hat{e}_k$ over $100$ episodes. 
As visible in Fig. \ref{fig:inference_energy_pendulum}, the task energy $e^*$ estimated by $\phi_{e_0}$ is lower than the energy that the system spends to perform the task when using the policy $\phi_{\infty}$. Indeed we appreciate how the agent $\phi_{e_0}$ learns to perform the task with lower energy consumption. 
Furthermore, the average return of the two agents  during the training  visible in Fig. \ref{fig:training_return_pendulum} converges to the same value (almost $2500$), while  $\phi_{\infty}$ can arrive to spend a level of energy that is almost $50$ times greater than the one spent in $\phi_{e_0}$, see Fig. \ref{fig:training_energy_pendulum}. Clearly, the episode terminations caused by the energy constraints (i.e., when the tank depletes), disturb the exploration during the training, which results in a slower convergence to the final policy.   Fig. \ref{fig:inference_error_pendulum} shows a comparison, in inference, of the two agents where a measure of  the position error $d_k = 1-\sin(\beta_k)$ is adopted as a success metric of the task. As visible both the agents converge to the same error, with a smoother trajectory for $\phi_{e_0}$.

 \section{Conclusions and Future Work}
\label{sec:Conclusion}

We introduced a framework merging the energy tank algorithm, used as a tool to passivize arbitrary control schemes, and reinforcement learning, representing the most versatile method to learn control policies along complex tasks. The presented procedures allow us to passivize any control policy, and to learn constructively passive policies, tackling the problem represented by a lack of optimization in the design phase of passive controllers.  
As future work, we are processing a throughout study of energy-aware robotics in the proposed neural framework, exploiting the energy tank architecture to embed metabolic and safety metrics, which depend on physical energy and power flows undergoing the robot. This goes beyond passive designs, which is a necessary step for many important tasks which require continuous energy injection (e.g., periodic locomotion), and because safety and passivity, although sometimes naively considered equivalent, are distinct concepts.


\renewcommand*{\bibfont}{\small}
\printbibliography

\vfill

\end{document}